\documentclass{article}
\usepackage{iclr2027_conference,times}
\usepackage{amsmath,amsfonts,amssymb,amsthm}
\usepackage{graphicx}
\usepackage{booktabs}
\usepackage{multirow}
\usepackage{xcolor}
\usepackage[colorlinks=true,linkcolor=blue,citecolor=blue,urlcolor=blue]{hyperref}
\usepackage{url}

\newcommand{\rr}{\mathbb{R}}
\newtheorem{proposition}{Proposition}
\newtheorem{definition}{Definition}

\title{When Does Recurrence Become an Algorithm?\\
Convergence Selection in Weight-Tied Looped Transformers}

\author{Tong Zhang \\
Fudan University\\
\texttt{tongzhang25@m.fudan.edu.cn}
\And
Junhao Hu\\
Peking University\\
\texttt{junhaohu@stu.pku.edu.cn}
\And
Yun Peng\\
Fudan University\\
\texttt{yunpeng4@sigsoft.org}
\And
Tao Xie\thanks{Corresponding author: \texttt{taoxie@pku.edu.cn}}\\
Peking University\\
\texttt{taoxie@pku.edu.cn}
}

\iclrfinalcopy
\begin{document}

\maketitle

\begin{abstract}
When does a weight-tied looped transformer---one block applied $T$
times---implement an actual algorithm? We answer with four findings from
controlled populations on group word problems. \textbf{(1) The budget
law:} free training installs a \emph{linear computation frontier}, a
mechanism that solves $v$ positions per loop, and prices its speed by the
training contract: $v \approx n_{tr}/T_{tr}$, exactly unity under
$T{=}n$ training---SGD consistently selects a frontier whose speed
matches the minimum the contract demands. Speed is an
experimenter-controllable knob, and granting more test-time loops than
ever trained rescues late positions at fixed input length. \textbf{(2) Architecture prior, not expressivity,
picks the algorithm:} standard-depth transformers learn parallel scans on
this task family; weight tying flips the selection to the serial
frontier, under global attention, even when positional addressing for a
$\log$-depth scan is supplied. \textbf{(3) The walls are not where
circuit complexity says:} $\mathrm{NC}^1$-completeness costs nothing
($A_5$ generalizes fully), while group order does ($S_5$'s
$120\times120$ operator deadlocks joint learning)---and an operator-first
curriculum dissolves the wall, revealing it as an optimization pathology,
not an impossibility. \textbf{(4) Mechanisms are portable objects:}
warm-starting across budget contracts transfers the algorithm in every
seed attempted---re-pricing its speed to the new contract---and bypasses
a seed lottery that defeats most from-scratch runs,
while imposing seriality through the input schedule fails where free
training succeeds---the mechanism can be moved, not mandated. These
results are invisible to standard instruments: skip and cross-step
similarity metrics provably saturate at the fixed points trained loops
converge to, measuring the trajectory's tail while the algorithm lives in
its head. We introduce a head instrument, the convergence-time scaling
$\tau(n,i)$, validate it causally by activation patching (damage cones
whose slope reproduces $v$), and show head instruments predict which
seeds generalize where tail metrics and path-independence scores do
not. The phenomena replicate on the public
easy-to-hard benchmark, including $4\times$ exact-match length
extrapolation.
\end{abstract}

\section{Introduction}

Depth-recurrent (looped) transformers reuse a single block $T$ times,
promising adaptive test-time compute: harder or longer inputs get more
loops~\citep{dehghani2019universal,geiping2025huginn,ouro2025}. The implicit
mechanistic story is that the loop is an \emph{algorithm}---that loop $t$
performs the $t$-th step of an iterative computation, so that running more
loops executes more steps. This story motivates training recipes, halting
rules, and interpretability agendas alike. Is it true?

The question splits into two: (i) what mechanism does a trained looped
transformer actually implement, and (ii) can our instruments even tell? We
show the second question gates the first. The natural tests---does skipping a
loop hurt? do successive loops use the same circuit?---are \emph{tail
instruments}: they probe the trajectory near its end state. But trained loops
converge, and a converged state is a fixed point of the loop map; near an
idempotent point, skipping is free and successive local circuits agree by
construction (Proposition~\ref{prop:tail}). A model that finishes its real
work in the first few loops and idles thereafter is indistinguishable, to
every tail instrument, from a model that never had discrete steps at all.
The algorithm, if there is one, lives in the head of the trajectory.

We make this concrete with a controlled population study: dozens of
weight-tied looped transformers on tasks with known one-shot
(constant-depth) solutions, spanning recipes (positional encoding, input
injection, loop schedule) and seeds. The population exhibits an
order-of-magnitude spread in length generalization---including a
near-full spread across seeds of a single recipe---yet \emph{every} tail
instrument is blind to it: single-step skip retention is exactly total
for every model with enough solved examples, cross-step Jacobian
alignment is uniformly high, and no cross-step similarity metric
correlates with generalization (Appendix~\ref{app:pilot}).

We then introduce a \emph{head} instrument. The convergence-time scaling
$\tau(n,i)$ is the first loop at which the decoded output at position $i$
reaches its final value and stays there, measured as a function of input
length $n$ on solved inputs. Its scaling class is a mechanism signature:
a constant $\tau$ is a shortcut (or pure refinement); $\tau\propto\log n$ is
a parallel scan; $\tau\propto n$ is serial, step-indexed computation. Unlike
tail metrics, $\tau$ has provable dynamic range: we validate it on a
streaming positive control---token $i$ is revealed only at loop $i$, so
serial computation is forced by construction---where it reads exactly
$\tau\propto n$ (slope $1.0$), while shortcut controls read flat.

Armed with a calibrated instrument, we characterize what free training
actually installs. Group word problems supply a task ladder with known
theory: solvable groups ($\mathbb{Z}_{60}$, $S_4$) admit constant-depth
shortcuts~\citep{liu2023shortcuts}, the $S_5$ and $A_5$ word problems are
$\mathrm{NC}^1$-complete~\citep{barrington1989}, and $\Theta(\log n)$
looped depth suffices in principle via an associative
scan~\citep{merrill2025depth}. Against every option this theory
menu offers---one-shot shortcut, $\log$-depth scan, fixed-point
refinement---trained models choose a fourth: a linear frontier whose
speed is a function of the training contract, whose walls track operator
scale rather than circuit class, and which moves between contracts by
warm-starting but cannot be installed by input scheduling.

\textbf{Contributions.}
(1)~The linear computation frontier and its \emph{budget law}: free
training installs a serial mechanism whose speed obeys
$v\approx n_{tr}/T_{tr}$ (exactly $1.00$ under $T{=}n$), with
per-position loop-extrapolation rescue at fixed length
(\S\ref{sec:boundary}).
(2)~Algorithm selection is set by architecture prior: weight tying flips
SGD's choice from the parallel scan of standard-depth transformers to the
serial frontier, with untied and positional-encoding controls
(\S\ref{sec:boundary}).
(3)~The learning walls are optimization objects, not complexity-theoretic
ones: $\mathrm{NC}^1$-hardness is free ($A_5$), operator scale is not
($S_5$), and an operator-first curriculum dissolves the wall; imposed
serial schedules fail where free training succeeds (\S\ref{sec:boundary}).
(4)~Mechanism portability: budget annealing transfers the algorithm
across loop contracts in $4/4$ seeds, bypassing the from-scratch seed
lottery (\S\ref{sec:anneal}).
(5)~The measurement theory making (1)--(4) visible: a saturation
proposition showing standard instruments are blind at fixed points, the
$\tau(n,i)$ head instrument with causal validation, and three 16-seed
races in which head instruments predict seed-level generalization while
tail metrics and path independence do not
(\S\ref{sec:pilot}--\ref{sec:tau}, \S\ref{sec:predict}).

\begin{figure}[t]
\centering
\includegraphics[width=\linewidth]{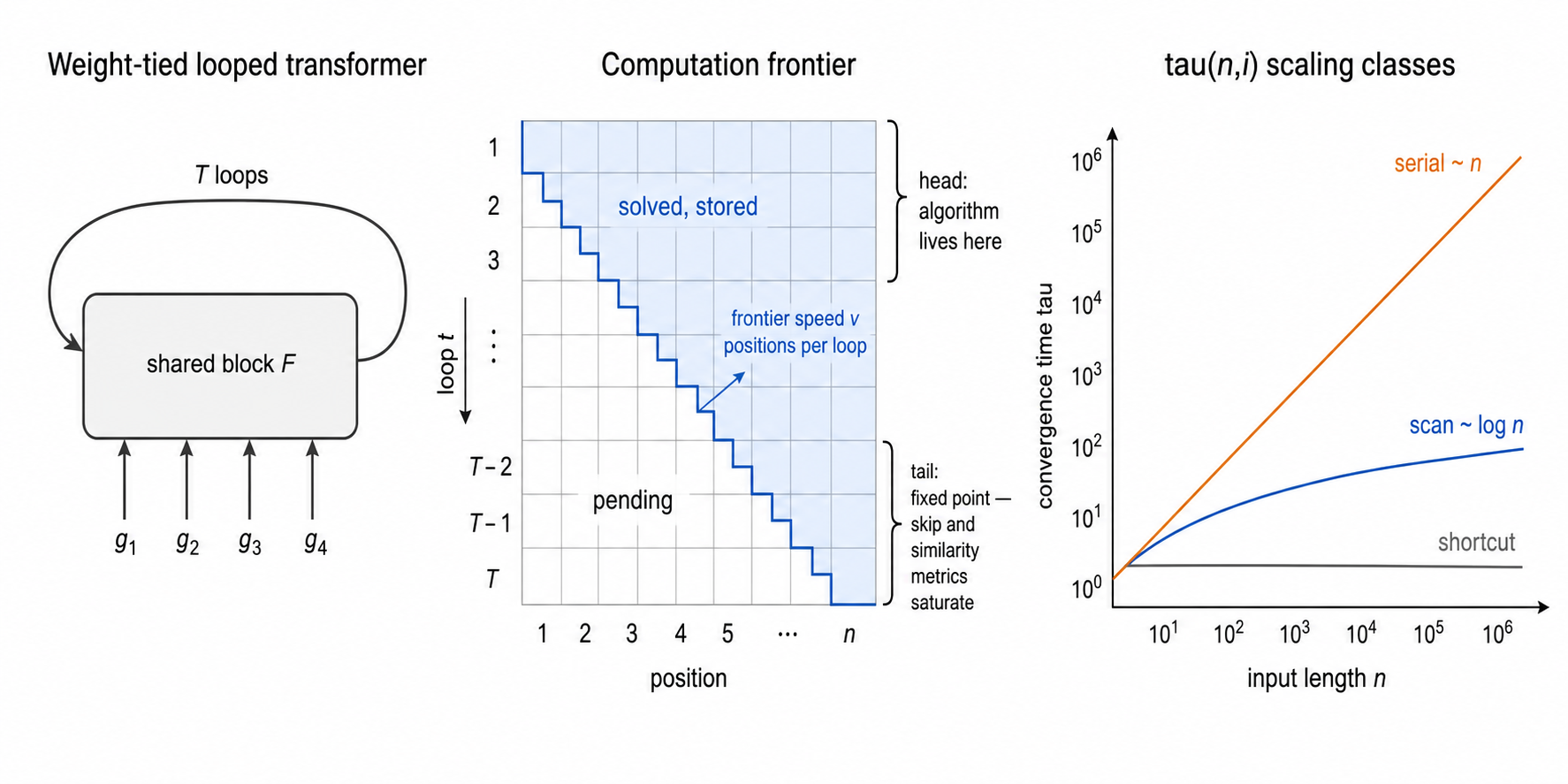}
\caption{Overview. \textbf{Left:} a weight-tied looped transformer applies
one shared block $T$ times. \textbf{Middle:} trained models solve
per-position tasks via a \emph{computation frontier} that advances a fixed
number of positions per loop; tail instruments only see the converged
region, where they provably saturate, while the algorithm lives in the
head of the trajectory. \textbf{Right:} the scaling class of the
convergence time $\tau(n,i)$ (flat / $\log n$ / linear) identifies the
mechanism.}
\label{fig:overview}
\end{figure}

\section{Related work}

\textbf{Looped and depth-recurrent transformers.}
Universal Transformers~\citep{dehghani2019universal} introduced weight tying
across depth; recent open models scale the idea
(Huginn-3.5B,~\citealp{geiping2025huginn}; Ouro,~\citealp{ouro2025}). Looped
transformers provably length-generalize on $n$-RASP-L
tasks~\citep{fan2024looped} and can emulate iterative
solvers~\citep{yang2023looped,gatmiry2024looped}; \citet{gatmiry2024looped}
show the global minimizer for in-context regression implements
preconditioned gradient descent---an inductive bias toward convergent
refinement, consistent with our selection result.

\textbf{Mechanistic accounts of loops.}
\citet{blayney2026mechanistic} document cyclic fixed points and
attention-pattern stabilization in Ouro and Huginn; their Prop.~4.2 (attention
similarity follows state convergence) anticipates our tail-saturation
analysis, which turns their observation into an instrument-validity
constraint. Lens-based studies of Huginn find limited latent-CoT
structure~\citep{lu2025latent}. \citet{kohli2026loop} patch hidden states at
(position, iteration) granularity in toy recurrent-depth models;
step-resolved \emph{data} attribution exists for looped
models~\citep{kaissis2026sdi}. None of these measure when loops carry
step-indexed computation; our $\tau$ instrument targets exactly that gap.
Feature-level circuit tracing~\citep{anthropic2025circuit} has been extended
to diffusion transformers with timestep-conditioned
transcoders~\citep{diffract2026}, but not to weight-tied loops; our results
suggest the loop-mechanism question must be settled first, since attribution
graphs inherit the same tail blindness.

\textbf{Convergence dynamics.}
Path independence---convergence to input-determined fixed points---predicts
upward generalization in deep equilibrium
models~\citep{anil2022path}; recall-style training explicitly promotes
it~\citep{bansal2022endtoend}. We adopt the asymptotic-alignment score as a
baseline and show the selection of convergence dynamics is a property of the
training pressure, not of recurrence per se.

\textbf{Complexity of transformers.}
Fixed-depth transformers sit in
$\mathrm{TC}^0$~\citep{merrill2023parallelism}; constant-depth shortcuts to
solvable-group automata exist and are learned, while non-solvable groups
($S_5$, $A_5$) admit none unless
$\mathrm{TC}^0=\mathrm{NC}^1$~\citep{liu2023shortcuts,barrington1989}.
Log-depth looped transformers close the gap via associative
scan~\citep{merrill2025depth}, and fixed-depth SSMs fail $S_5$ state
tracking~\citep{merrill2024illusion}. We use this ladder as the hardness
axis of the selection grid. Parity is expressible at constant
depth~\citep{chiang2022parity} yet hard to learn
length-generalizably~\citep{hahn2024sensitive}---expressibility and
learnability diverge, which is why our claims are about what training
\emph{selects}, not what architectures \emph{can do}.

\textbf{Imposed step-indexing.}
Iteration-wise supervision aligns loops with CoT steps~\citep{yu2025relay};
CLRS-style hints align processor iterations with algorithm
steps~\citep{velickovic2022clrs}, though hint-free models drift toward
parallel strategies~\citep{rodionov2023}. These are existence proofs that
loops \emph{can} be step-indexed when supervision imposes it; our streaming
arm imposes it architecturally, and the free arms measure whether anything
short of imposition suffices.

\section{Setup}
\label{sec:setup}

\textbf{Architecture.}
A weight-tied looped transformer applies a block $F_\theta$ (2 or 4
pre-norm transformer layers) $T$ times to a state $h^t\in\rr^{N\times d}$:
$h^{t+1}=F_\theta(h^t, e)$, where $e$ is the token embedding, optionally
re-injected each loop via a learned adapter (input
injection,~\citealp{bansal2022endtoend}). Readout is a linear head on
$\mathrm{LN}(h^T)$. Scales: \textsc{s} ($2$L, $d{=}128$), \textsc{m} ($2$L,
$d{=}256$), \textsc{l} ($4$L, $d{=}512$). No positional encoding (NoPE) in
the main grid; positional-encoding effects are factored in
Appendix~\ref{app:train}.

\textbf{Tasks.}
All tasks are per-position sequence-to-sequence, which blocks
answer-only shortcuts and gives every intermediate value a ground truth.
\emph{Pilot}: prefix parity, and LSB-first increment-with-carry (both
constant-depth shortcuttable). \emph{Selection grid}: prefix products over
$\mathbb{Z}_{60}$ (abelian), $S_4$ (non-abelian, solvable), and $S_5$
(non-solvable; $\mathrm{NC}^1$-complete word problem): input $g_1\dots g_n$,
target $p_i=g_1\!\circ\!\cdots\!\circ\! g_i$ at every $i$. The
$S_4$/$\mathbb{Z}_{60}$ controls separate non-commutativity from
non-solvability.

\textbf{Regimes.}
\emph{Free}: all tokens visible from loop 0; loop schedules
$T=\lceil\log_2 n\rceil+3$ (jittered) or fixed $T$; length curriculum
(required for $S_5$,~\citealp{merrill2025depth}). \emph{Streaming}: token
$i$ is revealed---embedded and attendable---only from loop $i{-}1$ on, so
serial computation is forced by construction; $T=n+2$. Streaming cells are
positive controls and metric ceilings, not discoveries: they calibrate
instruments and bound what emergence could look like.

\section{Tail instruments saturate at fixed points}
\label{sec:pilot}

\begin{proposition}[Tail blindness, informal]
\label{prop:tail}
Let $h^\star$ be a fixed point of the loop map $G(h)=F_\theta(h,e)$ reached
at loop $\tau^\star\le T-k$. Then (i) skipping any $k$ loops after
$\tau^\star$ leaves the output unchanged; (ii) the local Jacobians
$J_t=\partial h^{t+1}/\partial h^t$ agree for all $t\ge\tau^\star$; (iii) any
feature dictionary $z=E(h)$ and any attribution graph built from it agree
across loops $t\ge\tau^\star$. None of these quantities constrain the
computation performed at loops $t<\tau^\star$.
\end{proposition}

The proof is immediate from idempotence ($G(h^\star)=h^\star$) and is given
in Appendix~\ref{app:proof}, together with the quantitative version for
approximate fixed points ($\|G(h)-h\|\le\epsilon$). The consequence is an
instrument-validity constraint that, to our knowledge, has not been stated:
\emph{any} loop-mechanism claim based on skip tolerance, cross-step
similarity, or shared-dictionary attribution is uninformative unless the
measurement is localized to the head of the trajectory, $t<\tau^\star$.

\textbf{A population that tail instruments cannot order.}
We train 64 looped transformers (2 tasks $\times$ inject
$\{0,1\}$ $\times$ PE $\{\text{learned},\text{NoPE}\}$ $\times$
$T$-schedule $\{\text{prop},\text{fixed}\}$ $\times$ 4 seeds) to
in-distribution exact-match accuracy $1.0$; the population spans length
generalization $0.00$--$1.00$, including a $0.10$--$0.95$ spread
\emph{across seeds of a single recipe}. Every tail instrument fails to
order it: single-step skip retains every solved example in all 50
models with enough solved mid-OOD inputs (retention exactly $1.0000$);
cross-step Jacobian alignment spans $0.85$--$0.99$ and hidden-cosine /
attention similarity carry no relation to generalization (all tie-corrected
$|\rho|\le 0.32$; per-recipe details in Appendix~\ref{app:pilot}).
By Proposition~\ref{prop:tail} this is the expected reading for
trajectories that converge with slack---and measured convergence times sit
well inside the loop budget for all models. The pilot's information is not
``these models have no algorithm''; it is ``these instruments cannot
say.''

\section{The convergence-time instrument}
\label{sec:tau}

\begin{definition}[Convergence-time scaling]
For a solved input of length $n$, let $\hat y^t_i$ be the label decoded from
$h^t$ at position $i$. Define
$\tau(n,i)=\min\{k:\hat y^t_i=\hat y^T_i\ \forall t\ge k\}$, and
$\tau_{\max}(n)=\mathrm{median}_{\text{inputs}}\max_i \tau(n,i)$.
The \emph{scaling class} of a model is the growth of $\tau_{\max}(n)$:
flat (shortcut/refinement), $\Theta(\log n)$ (parallel scan), or
$\Theta(n)$ (serial).
\end{definition}

$\tau$ is forward-only (no gradients, no dictionaries, no patching), is
defined on solved inputs only (decoupling mechanism from competence), and
measures the head of the trajectory by construction: it reports when the
\emph{decodable output} stabilized, which upper-bounds when the
underlying computation finished but is not identical to it---the causal
cones of \S\ref{sec:boundary} are what license reading $\tau$
mechanistically, and we restrict mechanism claims to
cone-validated models. Per-position curves
$\tau(n,\cdot)$ further localize \emph{where} the frontier of computation
moves across loops, and on group tasks every intermediate $p_i$ has ground
truth, so head-of-trajectory decoding can be scored exactly.

\textbf{Calibration.}
The instrument must fire on a known serial mechanism and stay flat on a
known shortcut. On streaming parity (seriality forced by construction) it
reads slope $d\tau_{\max}/dn=1.00$. On fixed-$T$ parity models (bounded
budget) it reads $0.39$--$0.41$ locally within their narrow solved range,
bounded by $T$. Learned-PE models, which hit the positional wall, are
correctly reported as out of instrument range rather than
misclassified.\footnote{Complete calibration protocol, including the
frontier-resolved skip curve that replaces single-step skip (skip $k$
consecutive loops at every offset, stratified by distance to measured
$\tau$), in Appendix~\ref{app:calib}.}

\textbf{A first head-of-trajectory signal.}
On the pilot's seed-only recipe (parity, proportional $T$), $\tau$ slopes
order with generalization where every tail metric reads flat---the signal
quantified at population scale in \S\ref{sec:predict}.

\section{The linear computation frontier}
\label{sec:boundary}

\begin{figure}[t]
\centering
\includegraphics[width=\linewidth]{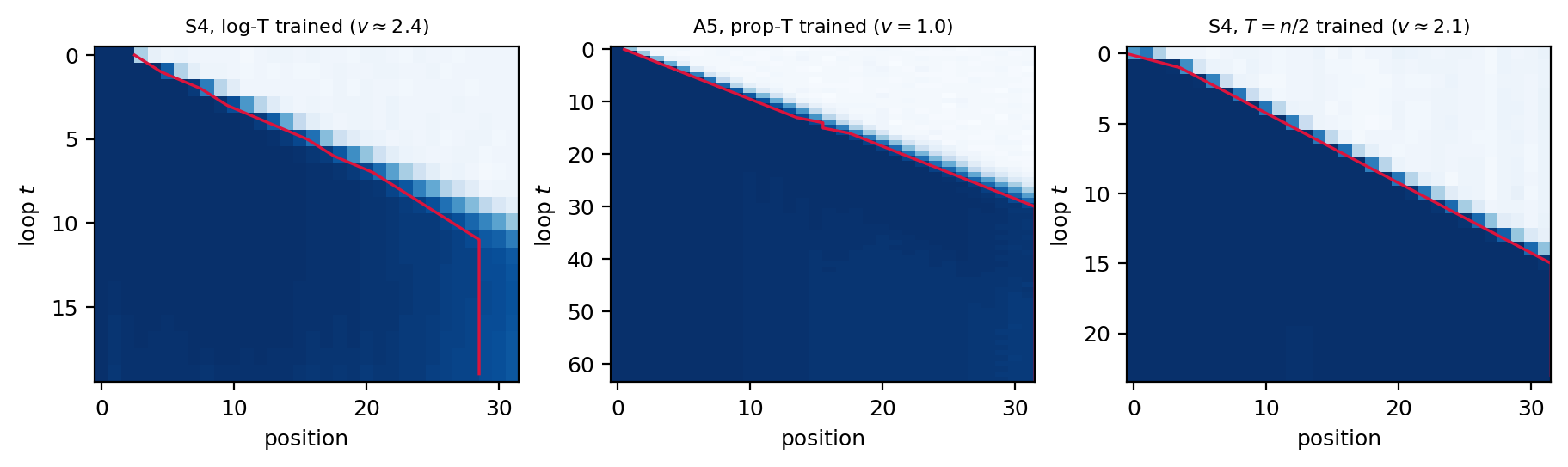}
\caption{Per-(loop, position) accuracy heatmaps at $n{=}32$ (blue: solved;
red line: contiguous solved prefix). Free training discovers a
\emph{linear computation frontier}: each loop extends the solved prefix by
a constant number of positions. The slope is a per-model constant across
lengths---a mechanism fingerprint---and is set by the training budget:
tight log-$T$ training yields $v\!\approx\!2.4$ (left), $T{=}n$ training
yields exactly $v{=}1.0$ (middle), $T{=}n/2$ yields $v\!\approx\!2.1$
(right).}
\label{fig:frontier}
\end{figure}

\textbf{Free training discovers frontier algorithms.}
On $S_4$ and $A_5$ prefix products, every model that learns the task
implements a moving computation frontier (Fig.~\ref{fig:frontier}): the
contiguous solved prefix advances $v$ positions per loop, with $v$
essentially constant across input lengths within a model
(Table~\ref{tab:speeds}). This is loop-indexed computation in the head of
the trajectory---not the constant-depth shortcut the solvable groups
admit, not a $\log n$ scan, and not fixed-point refinement. Because the
frontier is linear, budget-starved models fail \emph{late positions}
rather than whole inputs, and granting more loops than ever seen in
training rescues them per-position (Table~\ref{tab:rescue}): accuracy at
late positions of in-distribution lengths goes from near-zero at the
training budget to near-perfect at roughly twice the budget. Unlike prior
loop-extrapolation results on longer \emph{inputs}, this is rescue at
\emph{fixed} length, position-resolved, and quantitatively predicted by
$T \ge \tau(n,i)$. Table~\ref{tab:rescue} also exposes the other side of
the coin: pushing $T$ far beyond frontier completion \emph{degrades}
tight-budget models (overthinking), while loose-contract models remain
stable---the safe operating range of test-time compute is itself a
contract property.

\begin{table}[t]
\centering
\caption{Frontier speed and uniformity by training contract. $v$: median
frontier speed over $n\in\{16,32,64\}$ [seed range]; uniformity: median
fraction of $n{=}64$ solved by the frontier at $2.5\times$ the schedule
budget. $A_5$ pools all seeds including non-cracking ones (bimodal;
\S\ref{sec:predict}).}
\label{tab:speeds}
\small
\begin{tabular}{llccc}
\toprule
Group & Contract (demand $n/T$) & runs & $v$ & unif.@64 \\
\midrule
$S_4$ \textsc{s} & $T{=}2n$ \hfill ($0.5$) & 4 & $0.51$ [$.51,.52$] & $.55$ \\
$S_4$ \textsc{s} & $T{=}n$ \hfill ($1$) & 2 & $1.00$ [$1.00,1.00$] & $.98$ \\
$S_4$ \textsc{s} & $T{=}n/2$ \hfill ($2$) & 4 & $2.00$ [$1.93,2.13$] & $.19$ \\
$S_4$ \textsc{s} & $T{=}n/4$ \hfill ($4$) & 2 & $4.00$ [$3.60,4.00$] & $.00$ \\
$S_4$ \textsc{s} & log-$T$ \hfill (mixed) & 16 & $2.00$ [$1.86,2.42$] & $.33$ \\
$S_4$ \textsc{m} & log-$T$ \hfill (mixed) & 4 & $1.88$ [$1.86,1.88$] & $.41$ \\
$S_4$ \textsc{l} & log-$T$ \hfill (mixed) & 4 & $2.73$ [$2.09,3.33$] & $.09$ \\
$A_5$ \textsc{s} & $T{=}n$, horizon \hfill ($1$) & 16 & $0.96$ [$0.08,1.03$] & --- \\
$S_5$ \textsc{m} & op.-first, $T{=}n$ \hfill ($1$) & 2 & $1.00$ & $.56$ \\
$\mathbb{Z}_{60}$ \textsc{s} & log-$T$ \hfill (mixed) & 8 & $2.15$ [$1.82,5.17$] & --- \\
ETH parity & $T{=}n$ \hfill ($1$) & 4 & $\approx 1.5$ & --- \\
\bottomrule
\end{tabular}
\end{table}

\begin{table}[t]
\centering
\caption{Rescue surface and overthinking: mean per-token accuracy vs.\
test-time loops $T$ (group means over first three seeds). Tight-budget
(log-$T$) groups are rescued by $T$ beyond their training budget and then
degrade; the loose-contract operator-first $S_5$ improves monotonically to
frontier completion; $\mathbb{Z}_{60}$ and annealed models collapse
sharply past their operating range.}
\label{tab:rescue}
\small
\begin{tabular}{llccccccc}
\toprule
& & \multicolumn{7}{c}{test loops $T$} \\
Group & $n$ & 4 & 8 & 12 & 16 & 24 & 32 & 48 \\
\midrule
$S_4$ log \textsc{s} & 32 & .38 & .69 & .95 & \textbf{.97} & .91 & .85 & .78 \\
$S_4$ log \textsc{s} & 64 & .21 & .37 & .52 & .63 & \textbf{.66} & .60 & .55 \\
$S_4$ log \textsc{l} & 32 & .38 & .79 & \textbf{.96} & .86 & .74 & .69 & .64 \\
$S_5$ op-first \textsc{m} & 32 & .13 & .26 & .38 & .51 & .78 & \textbf{.99} & --- \\
$S_5$ op-first \textsc{m} & 64 & --- & .13 & .19 & .26 & .38 & .51 & .64 \\
anneal $A_5$ p$\to$l & 32 & .35 & .61 & \textbf{.76} & .76 & .52 & .26 & .13 \\
$\mathbb{Z}_{60}$ log \textsc{s} & 32 & .53 & \textbf{.75} & .74 & .74 & .45 & .12 & .05 \\
ETH prop & 32 & .52 & .63 & .69 & .73 & .75 & \textbf{.76} & .75 \\
\bottomrule
\end{tabular}
\end{table}

\textbf{The budget law.}
Frontier speed is set by the training contract (Fig.~\ref{fig:budget},
Table~\ref{tab:speeds}): across contracts spanning an $8\times$ range of
demand, the measured $v$ tracks $n_{tr}/T_{tr}$ with a power-law
exponent of $0.98\pm0.04$ (95\% CI; $R^2{=}0.99$ over 23 models at 8
demand levels; Fig.~\ref{fig:budgetfit}, left)---i.e.\ speed equals
demand, exactly unity under the $T{=}n$ contract---so SGD consistently
arrives at a frontier whose speed matches the \emph{minimum the contract
demands}, and $v$ is an experimenter-controllable knob
(Fig.~\ref{fig:concepts}, left). We state
this as a robust empirical regularity rather than a selection theorem:
demonstrating that faster-but-equally-trainable solutions exist at the
same contract (and are avoided) would require a constructive search over
solutions we have not performed; what the data establish is that the
demanded minimum is met, never exceeded, across an $8\times$ range of
consistent demands. Two
refinements. First, under the \emph{mixed} demands of the log-$T$
schedule, models compromise near the middle of the demand range and
under-deliver at the long end; a \emph{consistent} high demand
($T{=}n/4$ at every length) is met even at the smallest scale
(Table~\ref{tab:speeds})---the apparent speed ceiling is a product of
demand inconsistency, not capacity. Second, scale buys length robustness
more than speed: larger blocks run their frontiers slightly faster but
hold them together at longer lengths (Table~\ref{tab:rescue}), a
speed--robustness trade visible in the uniformity column of
Table~\ref{tab:speeds}.

\textbf{Cross-task form of the law.}
Replicating the demand sweep on three further tasks yields the law's
complete form: $v$ tracks demand within a task-dependent operating
band, $v \approx \mathrm{clip}(n_{tr}/T_{tr},\; v_{\mathrm{free}},\;
v_{\max})$. On $A_5$ the law is exact across an $8\times$ demand range
($0.5\!\to\!0.48$, $1\!\to\!1.00$, $2\!\to\!1.90$) with mild
under-delivery at the top ($4\!\to\!\approx\!2.5$: the expensive
$60{\times}60$ operator caps $v_{\max}$). Parity floors: measured $v$
never drops below $\approx\!1$--$1.5$ however loose the contract,
because attention composes cheap XOR operators over more than one
position at no extra cost ($v_{\mathrm{free}}$)---slowing down further
would not be simpler, so simplicity bias has nothing to push against;
this is also why loose-contract parity models carry idle tail loops
(the slack that saturates tail instruments, \S\ref{sec:pilot}).
$\mathbb{Z}_{60}$ obeys the law at loose contracts but \emph{escapes}
under pressure ($2\!\to\!4.2$, $4\!\to\!8.0$): the abelian task has a
parallel shortcut, and a tight contract pushes the model off the
frontier class entirely rather than to a faster frontier. The three
regimes---law-bound ($A_5$, $S_4$), floored (parity), escaped
($\mathbb{Z}_{60}$)---share one reading: training pressure sets the
mechanism's speed exactly when the task forbids anything cheaper.

\begin{figure}[t]
\centering
\includegraphics[width=\linewidth]{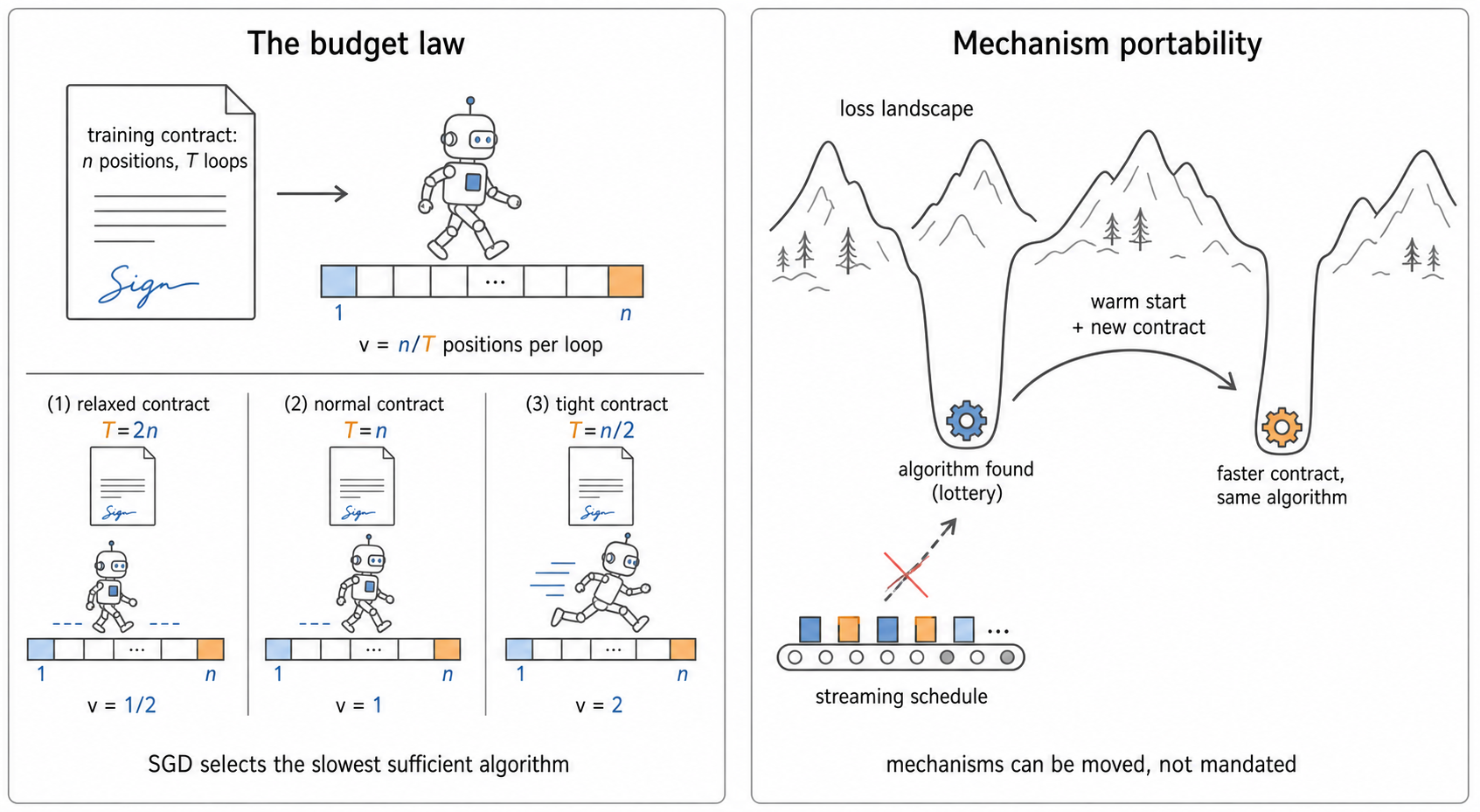}
\caption{Two organizing findings. \textbf{Left:} the budget law---the
training contract prices the frontier speed; the demanded minimum is
met, never exceeded. \textbf{Right:} mechanism portability---the
algorithm is found by a basin lottery, transfers across contracts by
warm-starting, and cannot be installed through the input schedule.}
\label{fig:concepts}
\end{figure}

\begin{figure}[t]
\centering
\begin{minipage}[t]{0.36\linewidth}
\vspace{0pt}
\includegraphics[width=\linewidth]{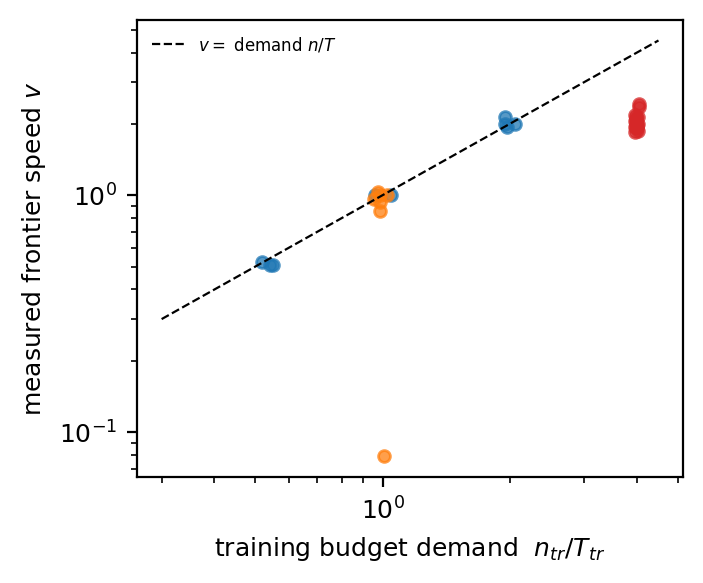}
\end{minipage}\hfill
\begin{minipage}[t]{0.60\linewidth}
\vspace{0pt}
\includegraphics[width=\linewidth]{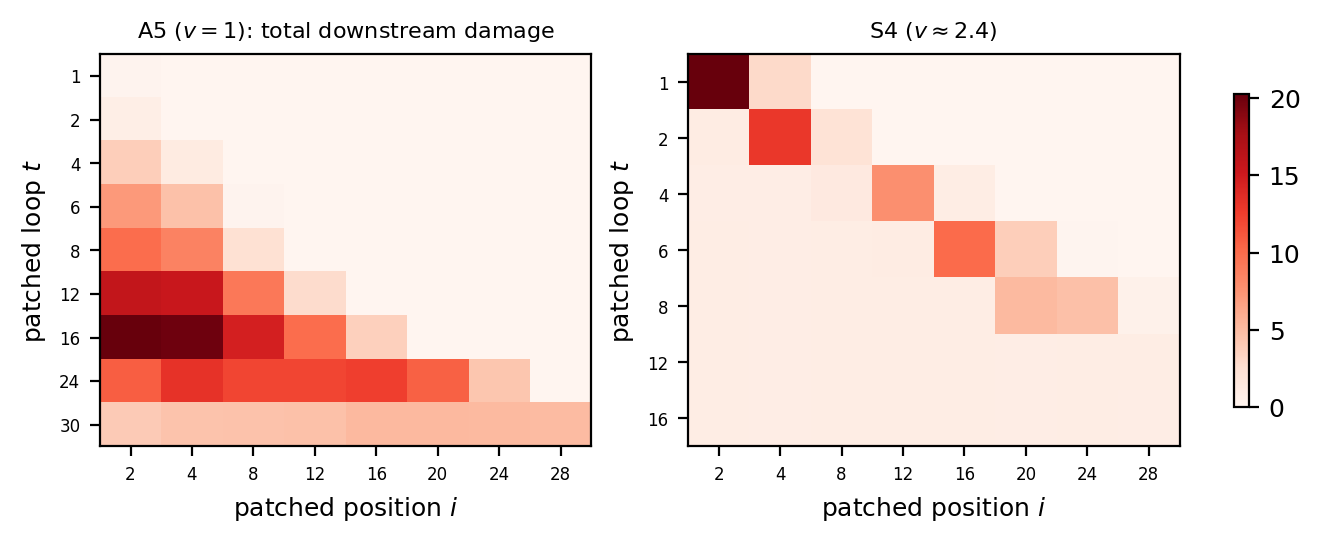}
\end{minipage}
\caption{\textbf{Left (budget law):} measured frontier speed tracks the
training demand $n_{tr}/T_{tr}$ (dashed: $v=$ demand) for consistent
demands; the log-$T$ cell (red, mixed demands averaging 4) compromises at
$v\!\approx\!2$. \textbf{Right (causal cones):} total downstream damage
from patching state at (position $i$, loop $t$). Damage is strictly
downstream (zero upstream in every cell), tracks the moving frontier, and
the cone slope reproduces each model's $v$ independently.}
\label{fig:budget}
\end{figure}

\textbf{The frontier is causal.}
Resampling the state at (position $i$, loop $t$) and rolling out the
remaining loops yields damage that is strictly downstream-in-position
(zero upstream damage in all $72$ grid cells for both models tested),
structured by the frontier: positions already passed by the frontier are
immune to upstream corruption, pending positions inherit it, and the
boundary between the two moves at each model's measured $v$
(Fig.~\ref{fig:budget}, right). $\tau$ is therefore a readout of a causal
computational object, not a decoding artifact. The cones also reveal two
storage phenotypes (Table~\ref{tab:repair}): one model class
\emph{repairs} corrupted stored values within a loop (negligible
self-damage behind the frontier), while the rest are store-once
(near-total, permanent self-damage)---a repair axis orthogonal to speed,
and notably rare: of the frontier models probed, only the star $A_5$
solution self-heals.

\begin{table}[t]
\centering
\caption{Repair phenotype: mean self-damage when the stored value at a
position already passed by the frontier is resampled (grid cells with
$t>i$). Low = self-healing; high = store-once.}
\label{tab:repair}
\small
\begin{tabular}{lcccc}
\toprule
& $A_5$ $T{=}n$ & $S_4$ log & $S_4$ $T{=}n/2$ & $S_5$ op-first \\
\midrule
self-damage behind frontier & $.02$ & $.95$ & $.96$ & $.99$ \\
\bottomrule
\end{tabular}
\end{table}

\textbf{What limits the frontier: operator scale, not circuit class.}
$S_5$ ($\mathrm{NC}^1$-complete) is never learned by free training under
the standard length curriculum---but the wall has nothing to do with
non-solvability (Fig.~\ref{fig:walls}). It is already present at
$n{=}2$: the time to master the single composition $g_1\!\circ\!g_2$
grows steeply with group order and, for $S_5$ at the smallest scale,
exceeds the entire training budget (Table~\ref{tab:opspeed}); one scale
up, the same operator is learned comfortably. The matched-order contrast
is decisive: $A_5$ is exactly as $\mathrm{NC}^1$-hard as $S_5$, yet
prop-$T$ training chains it to full length generalization at twice the
training length (Table~\ref{tab:anneal}, top). And the wall is a
\emph{curriculum} pathology, not an impossibility: an operator-first
curriculum---an extended $n{=}2$ stage that masters the composition
table before any chaining---cracks $S_5$ in \emph{every} seed attempted
($10/10$ across two machines, exact match $\ge 0.92$ at $2\times$
training length). Joint learning of the operator and the chain
deadlocks; sequenced learning does not, reliably. Non-solvability blocks nothing at these scales; the order of
the group sets the optimization cost of its operator, and curricula
settle it.

\begin{figure}[t]
\centering
\includegraphics[width=.9\linewidth]{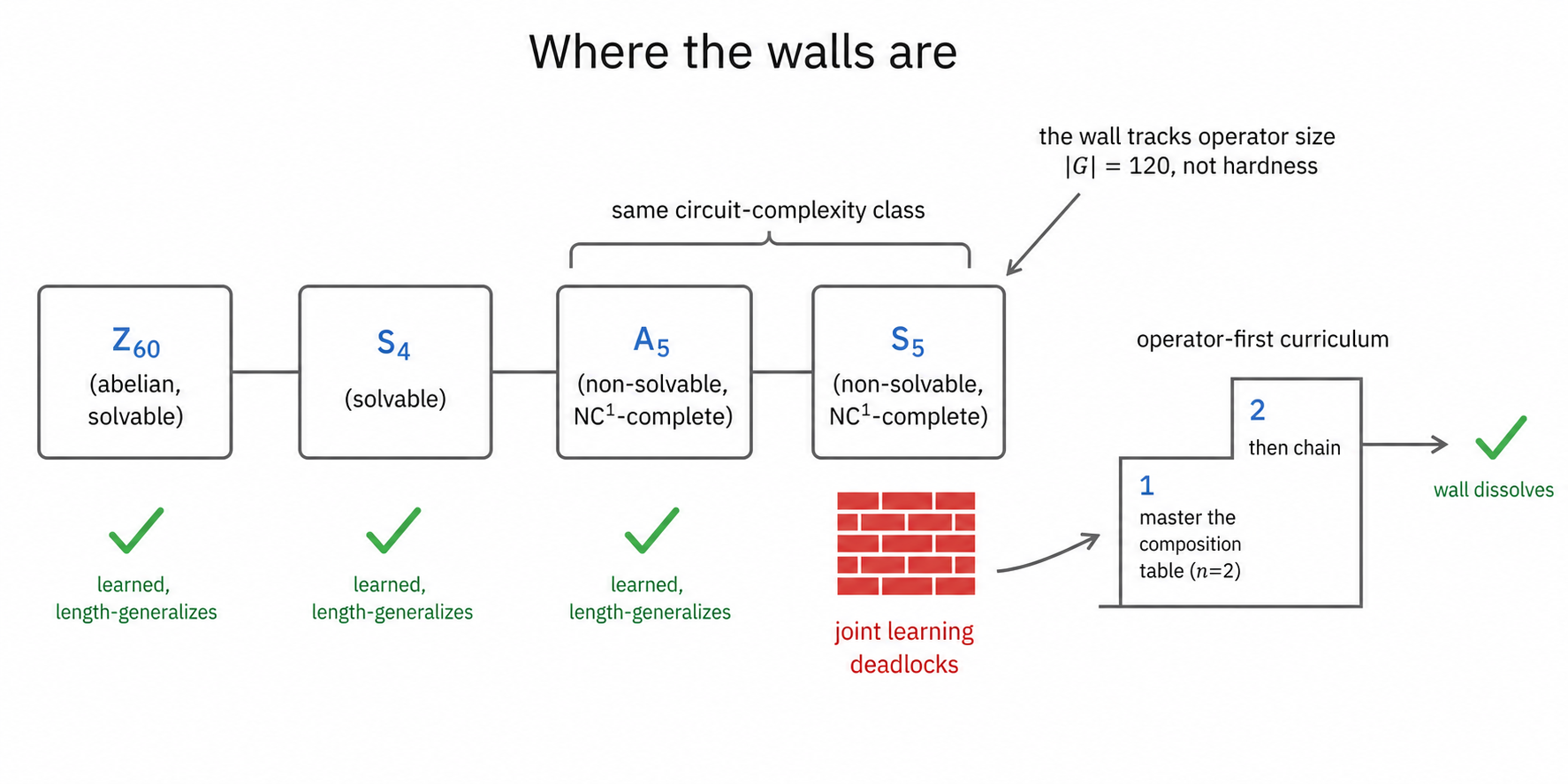}
\caption{The learning walls track operator size, not circuit-complexity
class: the non-solvable $A_5$ trains like the solvable groups, its
complexity-class twin $S_5$ deadlocks, and an operator-first curriculum
dissolves the deadlock.}
\label{fig:walls}
\end{figure}

\begin{table}[t]
\centering
\caption{Optimization cost of the group operator: training steps to
per-token accuracy $\ge .98$ on pair composition ($n{=}2$).
$S_4$/$A_5$ small-scale entries from exact-match diagnostics.}
\label{tab:opspeed}
\small
\begin{tabular}{lcccc}
\toprule
& $S_4$ ($|G|{=}24$) & $\mathbb{Z}_{60}$ & $A_5$ ($|G|{=}60$) & $S_5$ ($|G|{=}120$) \\
\midrule
scale \textsc{s} & $\sim$0.5k & 2.5k & $\sim$4k & $>$40k (never) \\
scale \textsc{m} & --- & 2.0k & --- & 6.5k \\
\bottomrule
\end{tabular}
\end{table}

\textbf{Imposed seriality does not help.}
Streaming input release (token $i$ revealed at loop $i$), the
by-construction serial schedule, trains perfectly on parity---with length
generalization to four times the training length---but \emph{fails} on
every group task including $A_5$, which free scheduling solves
(Table~\ref{tab:stream}). Adding per-loop supervision on the revealed
prefix does not rescue it. The frontier that free training invents is
easier to learn than externally imposed seriality: mechanism selection
cannot be forced through the input schedule, which motivates the
intervention that works (\S\ref{sec:anneal}).

\begin{table}[t]
\centering
\caption{Mechanism cannot be mandated through the input schedule:
outcome by task $\times$ schedule (\checkmark\ = learns and
length-generalizes; $\times$ = fails to learn beyond trivial positions).}
\label{tab:stream}
\small
\begin{tabular}{lccc}
\toprule
& free & streaming & streaming + per-loop sup. \\
\midrule
parity & \checkmark & \checkmark & --- \\
$A_5$ & \checkmark\ (lottery) & $\times$ & $\times$ \\
$S_5$ & $\times$ (op.\ wall) & $\times$ & $\times$ \\
$S_5$, operator-first & \checkmark & --- & --- \\
\bottomrule
\end{tabular}
\end{table}

\textbf{Controls: the tying axis isolated.}
We train a tying-degree ladder at fixed effective depth 12: the $D$
layers are drawn from $G$ weight groups cycled ($G{=}1$: our fully tied
loop; $G{=}12$: a standard untied transformer; intermediate $G$: partial
tying), in two regimes---constant width (parameters grow $10\times$ with
$G$) and parameter-matched (width shrunk to hold parameters fixed)
(Table~\ref{tab:tying}). The regimes disagree, and the disagreement is
the point: with free parameters, untying appears mildly beneficial; with
parameters matched, the fully untied model is \emph{worst} at length
extrapolation, and on $A_5$ every configuration with $G\ge3$ fails to
learn the task at all ($9/9$ seeds) while tied configurations learn it.
Weight reuse is not a resource handicap but a load-bearing inductive
bias---necessary for learnability on the harder group, and for
extrapolation on the easier one---consistent with standard transformers
learning parallel-scan mechanisms on this family~\citep{liu2023shortcuts}.
Learned absolute positions (giving the model the addressing a
$\log$-depth scan would need) leave the mechanism unchanged: the frontier
persists, with the same rescue signature.

\begin{table}[t]
\centering
\caption{Tying-degree control at fixed effective depth 12 ($G$ weight
groups cycled; mean per-token accuracy at $n{=}64$, 3 seeds per cell).
With parameters matched, untying hurts extrapolation ($S_4$) and destroys
learnability ($A_5$, $G\ge3$ at chance).}
\label{tab:tying}
\small
\begin{tabular}{llcccccc}
\toprule
& & $G{=}1$ (tied) & $G{=}2$ & $G{=}3$ & $G{=}4$ & $G{=}6$ & $G{=}12$ (untied) \\
\midrule
\multirow{2}{*}{$S_4$} & const.\ width & .21 & .27 & .30 & .30 & .32 & .26 \\
& param-matched & .21 & \textbf{.28} & .27 & .25 & .25 & \textbf{.15} \\
\midrule
\multirow{2}{*}{$A_5$} & const.\ width & .10 & .23 & .17 & .25 & .23 & .15 \\
& param-matched & .10 & \textbf{.15} & \textbf{.05} & .05 & .05 & .05 \\
\bottomrule
\end{tabular}
\end{table}

\section{Annealing the budget transfers the mechanism}
\label{sec:anneal}

If the frontier is real and its speed is priced by the budget, mechanisms
should transfer across budgets by warm-starting---and transfer should
bypass the seed lottery that afflicts training from scratch (across
$22$ seeds on two machines, roughly a third crack chained $A_5$; the
rest never pass the second position---a bimodal outcome stable across
hardware and run order). Both predictions hold, in every seed attempted
(Table~\ref{tab:anneal}). Warm-starting from a $T{=}n$ solution and
re-training under the tight log-$T$ contract meets the new contract;
the reverse anneal fully solves lengths that no from-scratch run of
either schedule achieved. Most tellingly, the transferred mechanism is
\emph{re-priced}: annealing a $v{=}1$ solution into the tight contract
doubles its measured frontier speed, and the reverse anneal slows a
fast frontier back toward unity (Table~\ref{tab:anneal})---the budget
law acts on warm-started mechanisms just as it does on fresh ones, but
without re-running the lottery. Once any budget regime finds the
algorithm, the algorithm---not the budget---is the hard-won object, and
it retrains cheaply into any contract (Fig.~\ref{fig:concepts}, right).
The annealed models inherit the tight contract's narrow operating range
(Table~\ref{tab:rescue}), a caveat for deployment.

\begin{table}[t]
\centering
\caption{Budget annealing: warm-start across loop contracts. ``Cracked''
= learns chained composition with length generalization. The frontier
speed of the transferred mechanism is re-priced by the new contract.}
\label{tab:anneal}
\small
\begin{tabular}{llccc}
\toprule
Arm & contract & cracked & $v$ before & $v$ after \\
\midrule
$A_5$ from scratch & $T{=}n$ & 8/22 seeds & --- & $\approx 1.0$ \\
$A_5$ from scratch & log-$T$ & bimodal, weaker & --- & $\approx 2$ \\
$A_5$ anneal & $T{=}n \to$ log & 4/4 & $1.0$ & $2.0$ \\
$S_4$ anneal & $T{=}n \to$ log & 4/4 & $1.0$ & $2.4$ \\
$S_4$ anneal & log $\to T{=}n$ & 4/4 & $2.4$ & $1.2$ \\
$S_5$ operator-first & $T{=}n$ & 10/10 & --- & $1.0$ \\
\bottomrule
\end{tabular}
\end{table}

\begin{figure}[t]
\centering
\includegraphics[width=.42\linewidth]{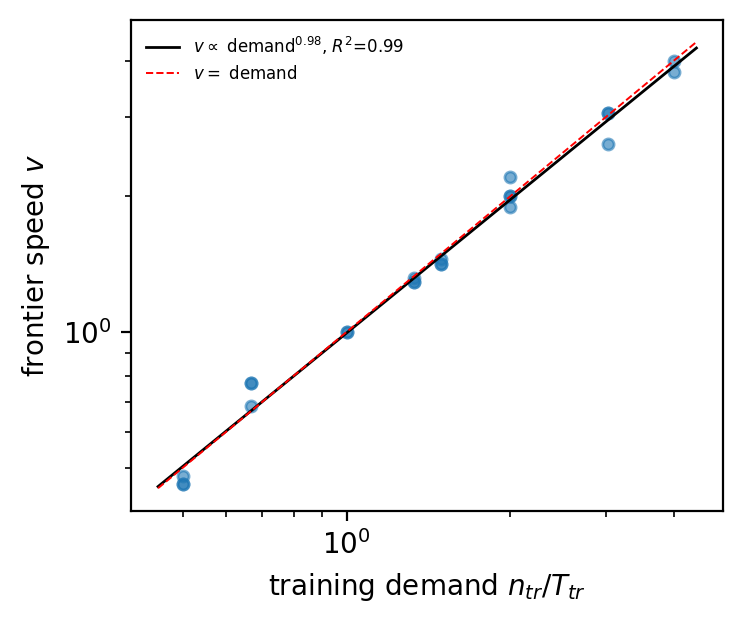}\hfill
\includegraphics[width=.42\linewidth]{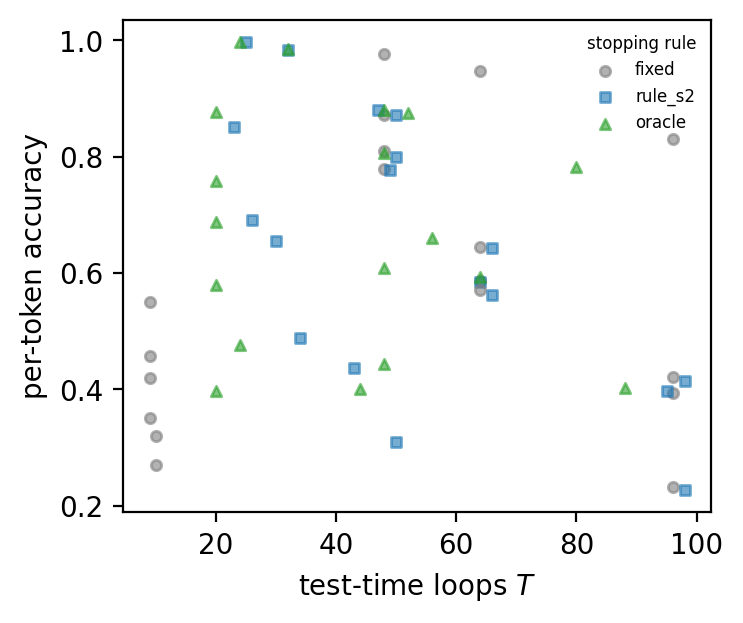}
\caption{\textbf{Left:} the budget law densified---frontier speed vs.\
training demand over 23 models spanning $8\times$ in demand; the fit is
$v\propto\text{demand}^{0.98\pm0.04}$, $R^2{=}0.99$ (dashed: $v=$
demand). \textbf{Right:} the law as a halting rule---stopping at
$T^\star{=}\lceil n/\hat v\rceil$ with $\hat v$ estimated in
distribution (blue) nearly matches an oracle $T$ sweep (green) and far
exceeds the training schedule's own fixed $T$ (grey), while avoiding the
overthinking collapse of over-long $T$.}
\label{fig:budgetfit}
\end{figure}

\textbf{A halting rule falls out of the law.}
Because $v$ is estimable in distribution and the frontier completes at
$T\!\approx\!n/v$, the budget law is directly actionable: stopping at
$T^\star=\lceil n/\hat v\rceil$ (with $\hat v$ read off training
lengths) recovers near-oracle accuracy at test lengths and avoids
overthinking (Fig.~\ref{fig:budgetfit}, right)---e.g.\ on a tight-budget
$S_4$ model at $n{=}48$ the training schedule's fixed $T$ scores far
below the rule, which nearly matches the best achievable over a full $T$
sweep. Test-time compute for looped models has a principled setting, not
a hyperparameter to tune.

\section{Head instruments predict generalization; tail instruments do not}
\label{sec:predict}

We race metrics measured \emph{strictly in distribution}
($n\le n_{tr}{=}32$: frontier speed at $n{=}16$ and $32$, frontier
completion at $n{=}32$ under the schedule budget, $\tau$ slope over ID
lengths, and the tail baselines) against length generalization measured
\emph{strictly out of distribution} (mean per-token accuracy at
$n\in\{64,128\}$, schedule-matched $T$, recomputed from checkpoints),
across three $16$-seed single-recipe populations---no recipe confounds,
no predictor sees any length beyond training, tie-aware Spearman,
partial correlations controlling adjacent-loop cosine, and median-split
AUROC.

\begin{table}[t]
\centering
\caption{Predicting OOD generalization ($n\in\{64,128\}$) from strictly
in-distribution measurements ($n\le 32$), $16$ seeds per cell. Entries:
Spearman $\rho$ (median-split AUROC). Head instruments: frontier speed
$v_{32}$ and frontier completion at $n{=}32$ under the schedule budget;
tail baselines: adjacent-loop cosine, Jacobian alignment; AA = asymptotic
alignment~\citep{anil2022path}. Dashes: metric undefined for that cell
(failed seeds have no frontier; loose contracts leave no ID speed
variance).}
\label{tab:race}
\small
\begin{tabular}{lccc}
\toprule
Metric [ID, $n\le32$] & $S_4$ log-$T$ & $A_5$ prop-$T$ & parity prop-$T$ \\
\midrule
frontier speed $v_{32}$ [head] & $\mathbf{+.87}\;(.95)$ & --- & $+.38\;(.62)$ \\
frontier completion@32 [head] & $+.83\;(.94)$ & $\mathbf{+.90}\;(1.00)$ & $\mathbf{+.58}\;(.84)$ \\
cosine (adjacent) [tail] & $-.02\;(.44)$ & $-.77\;(.03)$ & $+.01\;(.37)$ \\
Jacobian align.\ [tail] & $-.61\;(.10)$ & $-.79$ & $+.11\;(.44)$ \\
AA / path indep. & $+.31\;(.60)$ & $-.66\;(.14)$ & $+.21\;(.65)$ \\
\bottomrule
\end{tabular}
\end{table}

In every cell the best predictor is a head instrument
(Table~\ref{tab:race}), under a protocol in which no predictor sees any
length beyond training. The consistent picture: frontier
\emph{completion at the training frontier}---did the mechanism finish
its work within the contract, measured entirely in distribution---is the
dominant predictor everywhere it is defined (AUROC up to $1.00$), and
ID frontier \emph{speed} adds signal exactly where the budget law says
it should: in the tight-budget cell, where speed is under pressure and
varies across seeds, but not in loose-contract cells, where the contract
pins it. Tail metrics are not uniformly blind---in the sharply bimodal
$A_5$ cell, adjacent cosine separates ``still computing'' from
``collapsed''---but they carry no signal \emph{within} a mechanism
class, where head instruments retain it. The path-independence score
never wins a cell. Failed seeds having no measurable frontier is itself
the strongest ID signal in the bimodal cells: the presence of a frontier
at training lengths is close to a sufficient statistic for OOD fate. Where the negative Jacobian correlations are significant, they point
the same way: seeds whose loop maps align early (converge early) generalize
worse.

\section{Open-benchmark validation}
\label{sec:eth}

On the community easy-to-hard prefix-sums benchmark
\citep{schwarzschild2021can} (official 32-bit training split, official
512-bit test set), the same phenomena replicate without modification
(Table~\ref{tab:eth}). Under the $T{=}n$ contract, successful seeds learn
a linear frontier with speed constant across lengths and extrapolate to
four times the training length at near-perfect \emph{exact-match}
accuracy; the seed lottery replicates; and the budget law predicts the
log-$T$ arm's total failure at $512$ bits exactly (demand more than an
order of magnitude above the learned speed). The frontier breaks between $128$ and $256$
bits (Table~\ref{tab:eth}), well short of the $512$-bit extrapolation
that convolutional Deep Thinking networks achieve on this
dataset~\citep{bansal2022endtoend}---an
honest asymmetry with a mechanistic reading: convolutional locality
\emph{forces} a linear frontier and shields it from long-range
interference, whereas the frontier our models \emph{learn} inside global
attention inherits attention's length fragility. The mechanism class
transfers to the open benchmark; its range is architecture-limited. We
emphasize this section is mechanism validation, not a state-of-the-art
generalization claim---convolutional architectures remain stronger on
this benchmark, for the architectural reason above.

\begin{table}[t]
\centering
\caption{Easy-to-hard prefix sums (train 32-bit official split):
exact-match accuracy by test length, per seed, $T{=}n$ contract; log-$T$
contract fails everywhere beyond training length as the budget law
predicts.}
\label{tab:eth}
\small
\begin{tabular}{lccccc}
\toprule
Arm & $n{=}32$ & $n{=}64$ & $n{=}128$ & $n{=}256$ & official 512 \\
\midrule
$T{=}n$, seed 2 & $1.00$ & $1.00$ & $1.00$ & $0$ & $0$ \\
$T{=}n$, seed 3 & $1.00$ & $1.00$ & $0.96$ & $0$ & $0$ \\
$T{=}n$, seeds 0/1 (lottery) & $1.00$ & $0$ & $0$ & $0$ & $0$ \\
log-$T$, all seeds & $1.00$ & $0$ & $0$ & $0$ & $0$ \\
DT conv-RNN \citep{bansal2022endtoend} & \multicolumn{5}{c}{solves 512 (locality-forced frontier)} \\
\bottomrule
\end{tabular}
\end{table}

\section{Discussion}

\textbf{What the loop buys.}
The loop's contribution is neither ``more refinement'' nor ``more
algorithm steps'' generically: free training installs a linear frontier
whose speed is priced by the training contract, whose robustness varies
by seed along an orthogonal repair axis, and whose completion---not
convergence---is what test-time loops purchase. This reframes halting and
loop-extrapolation design: the useful stopping signal is frontier
completion ($T \ge n/v$), predictable per input length, rather than
state convergence. It equally reframes interpretability for
depth-recurrent LMs: instruments that probe converged states inherit
Proposition~\ref{prop:tail}'s blindness, and should be replaced or
augmented by head-of-trajectory measurements.

\textbf{Localization, transport, and editing.}
Appendix~\ref{app:local} tests our claims against the localization
literature's standards: frontier-state transplants succeed or fail
exactly as the repair phenotype predicts, module-resolved cones place
the frontier's routing in the first attention layer, and the
localization-vs-editing dissociation of \citet{hase2023does} becomes an
architectural theorem under weight tying, whose smearing signature we
measure. All our localization claims accordingly live, and are
validated, in activation space.

\textbf{Selection, not expressivity.}
Standard-depth transformers learn parallel scans on this task family and
fail to find sequential mechanisms~\citep{liu2023shortcuts}; our untied
control degrades similarly. Weight tying flips the selection to a serial
frontier even when positional addressing for a $\log$-depth scan is
supplied. Which algorithm SGD finds is set by the architecture prior and
the budget contract---not by what the architecture can express.

\textbf{Scope.}
Claims concern what end-to-end training selects in weight-tied loops on
per-position algorithmic sequence tasks at $\le 25$M parameters. The
$S_5$ result is a statement about optimization cost of the composition
operator at $|G|{=}120$, not an $\mathrm{NC}^1$ impossibility claim; the
$\mathrm{TC}^0\!\ne\!\mathrm{NC}^1$ separation is conditional and forces
only super-constant depth. Pretrained depth-recurrent LMs (Huginn, Ouro)
are out of scope here; extending $\tau$ to them requires resolving
tokenization and format confounds and is left to future work.

\bibliography{refs}
\bibliographystyle{iclr2027_conference}

\appendix
\section{Localization, transport, and editing under weight tying}
\label{app:local}

The survey lines that dominate interpretability---neuron
localization~\citep{dai2022knowledge,wang2022skill,tang2024language},
causal-tracing-guided editing~\citep{meng2022rome}, and circuit
discovery~\citep{conmy2023acdc,yao2024knowledge}---all index structure by
\emph{parameter address} (layer, neuron, head). Weight tying voids that
index: one address serves every loop. Three experiments show what
replaces it.

\textbf{State transplant follows the repair phenotype.}
Splicing input $A$'s frontier states into input $B$'s run at matched
(position, loop)---the state-level analogue of function-vector
transport~\citep{todd2024function}---has opposite fates in the two
storage phenotypes of Table~\ref{tab:repair}: in the store-once $S_4$
model the transplanted prefix \emph{sticks} (retained nearly perfectly,
essentially never reverting), while the self-healing $A_5$ model
\emph{repairs} the transplant away, reverting almost every position to
its own input's answer (Table~\ref{tab:transplant}). Together with
\S\ref{sec:anneal} this completes a three-way separation:
\emph{states} transport where storage is passive; the \emph{mechanism}
transports universally in weight space (annealing); and \emph{schedules}
transport nothing.

\begin{table}[t]
\centering
\caption{Frontier-state transplant ($A$'s prefix states spliced into
$B$'s run at the frontier, loop $t{=}16$): retention vs.\ repair is
governed by the storage phenotype of Table~\ref{tab:repair}.}
\label{tab:transplant}
\small
\begin{tabular}{lccc}
\toprule
Model & keeps $A$ prefix & reverts to $B$ & phenotype \\
\midrule
$S_4$ log-$T$ & $.97$--$1.00$ & $.05$ & store-once \\
$A_5$ $T{=}n$ & $.02$ & $.99$ & self-healing \\
\bottomrule
\end{tabular}
\end{table}

\textbf{Module-resolved cones.}
Splitting the (position, loop) patch by module localizes the frontier's
machinery within the tied block: the first layer's attention carries
nearly all cross-position damage (an order of magnitude above the second
layer's, which is essentially inert), with both MLPs intermediate---the
routing that advances the frontier lives in one attention layer, a
module-level claim no fixed-depth circuit method can even express, since
its nodes are parameter addresses rather than (module, loop) pairs.

\textbf{The localization--editing dissociation is architectural.}
\citet{hase2023does} showed that causal-tracing localization does not
predict where weight edits succeed. Under weight tying this dissociation
is a theorem---any weight edit hits every loop---and we measure its
signature: perturbing FFN rows selected by gradient attribution at a
target band of loops produces damage that accumulates monotonically with
loop index and shows \emph{zero} band specificity (inside-band and
outside-band damage statistically identical). Activation-space
localization cannot inform weight-space editing here even in principle;
this is why all our localization claims live, and are validated, in
activation space (\S\ref{sec:tau}). Conversely, it sharpens what our
instruments mean: $\tau$ and the cones localize \emph{when and where
computation happens}, a question that remains well-posed precisely
because it is not a claim about parameter storage.

\section{Tail blindness: formal statement}
\label{app:proof}

\begin{proposition}[Exact form]
Let $G:\rr^{N\times d}\to\rr^{N\times d}$ be the loop map
$G(h)=F_\theta(h,e)$ for a fixed input embedding $e$, and suppose the
trajectory $h^{t+1}=G(h^t)$ reaches a fixed point: $h^t=h^\star$ for all
$t\ge\tau^\star$, with $\tau^\star\le T-k$. Then:
(i) for any $j\ge\tau^\star$ and any $k\le T-j$, skipping loops
$j,\dots,j{+}k{-}1$ (i.e., outputting
$G^{T-j-k}(h^{j})$ in place of $G^{T-j}(h^{j})$) leaves the final state,
and hence the decoded output, unchanged;
(ii) $J_t:=\partial G/\partial h\,\big|_{h^t}=J_{\tau^\star}$ for all
$t\ge\tau^\star$, so any similarity functional of local Jacobians is
constant on the tail;
(iii) for any feature encoder $E$ and any attribution graph constructed
from activations $z^t=E(h^t)$ and linearizations at $h^t$, all
tail-indexed quantities coincide for $t\ge\tau^\star$.
None of (i)--(iii) constrains $\{h^t\}_{t<\tau^\star}$ or the map's
behavior off the trajectory.
\end{proposition}

\begin{proof}
(i) $G(h^\star)=h^\star$ implies $G^{m}(h^{j})=h^\star$ for every
$m\ge 0$ when $h^{j}=h^\star$; both branches output $h^\star$.
(ii)--(iii) are immediate since all quantities are functions of the
(constant) tail state. The final claim holds because the statements
quantify only over $t\ge\tau^\star$.
\end{proof}

\begin{proposition}[Approximate form]
If instead $\|G(h^t)-h^t\|\le\epsilon$ for $t\ge\tau^\star$ and $G$ is
$L$-Lipschitz on a neighborhood containing the tail, then skipping $k$
tail loops perturbs the final state by at most
$\epsilon\sum_{i=0}^{k-1}L^{\,T-\tau^\star-i}$, and tail Jacobians differ
by at most $\mathrm{Lip}(J)\cdot\epsilon/(1-L)$ when $L<1$. For
contractive tails ($L<1$) both bounds are $O(\epsilon)$: near-converged
trajectories are indistinguishable from converged ones at instrument
precision.
\end{proposition}

Empirically, all 64 pilot models show relative update norms $<10^{-2}$
over the final third of the loop budget on solved inputs. Small update
norms are consistent with (though do not strictly prove) a contractive
tail; we therefore treat the proposition as explaining the observed
saturation rather than as a verified premise, and note that models
outside the near-converged regime (e.g., the diverging
$\mathbb{Z}_{60}$ models of Table~\ref{tab:rescue}) are explicitly
outside its scope---there, tail instruments can be informative.

\section{Calibration protocol and pilot details}
\label{app:calib}

\textbf{$\tau$ calibration.} The instrument must fire on a known serial
mechanism and stay bounded on known shortcuts. Streaming parity
(seriality by construction): $d\tau_{\max}/dn = 1.00$. Fixed-$T$ parity
(budget-bounded): local slopes $0.39$--$0.41$, correctly capped by $T$.
Proportional-$T$ parity: slopes $0.50$--$0.70$, ordered with seed-level
generalization. Learned-PE models, which fail beyond training length,
yield fewer than three solved lengths and are reported as out of
instrument range rather than assigned a class. $\tau$ requires solved
inputs at $\ge 3$ lengths; for tight-budget models that fail exact match
at long $n$, frontier speed $v$ (measured per length) substitutes, as in
\S\ref{sec:predict}.

\textbf{Frontier-resolved skip.} The single-step skip of the pilot is
replaced by: skip $k\in\{1,2,4\}$ consecutive loops at every start
offset, retention measured per example on solved inputs at the largest
solved length, stratified by offset relative to the measured frontier.
Frontier models show retention $1.0$ only for offsets past frontier
completion; refit-style models show it everywhere.

\label{app:pilot}
\textbf{Pilot population.} 64 models (Appendix~\ref{app:train} for
recipes), all at in-distribution exact match $1.0$. Tail readings:
single-step skip retention exactly $1.0000$ for all $50$ models with
$\ge 8$ solved examples at $n{=}24$; adjacent-loop Jacobian probe
alignment $0.85$--$0.99$; hidden cosine $0.59$--$0.94$; attention
similarity $0.76$--$0.99$. Tie-corrected Spearman against length
generalization, within task: all $|\rho|\le 0.32$ (cosine $\le 0.23$),
within-recipe seed contrasts flat. Convergence-time slopes over the same
population span $0.39$--$0.70$ and order with generalization in the
seed-only recipe (\S\ref{sec:predict}).

\section{Training and experiment details}
\label{app:train}

\textbf{Architecture.} Pre-norm transformer blocks, $4d$ FFN, GELU;
weight-tied across loops; input injection via a learned linear adapter on
$[h;e]$; readout head on $\mathrm{LN}(h^T)$. Scales: \textsc{s}
(2 layers, $d{=}128$, 4 heads, $\sim$1.6M params), \textsc{m} (2L,
$d{=}256$, 8H), \textsc{l} (4L, $d{=}512$, 8H, $\sim$25M); untied control
\textsc{u12} (12 distinct layers, $T{=}2$, depth-matched to \textsc{s} at
$T{=}12$). NoPE throughout except the learned-PE arms.

\textbf{Training.} AdamW, lr $3{\times}10^{-4}$
($2{\times}10^{-4}$ for \textsc{l}), weight decay $0.01$, grad clip
$1.0$, batch 256 (128 for \textsc{l}). Length curriculum with stages
$n_{\max}\in\{4,8,16,32\}$ (6k--12k steps each), promotion at per-token
accuracy $\ge 0.98$; lengths sampled uniformly within a stage. Loop
schedules: log ($T{=}\lceil\log_2 n\rceil{+}3$, jitter $\pm1$), prop
($T{=}\lceil n\cdot m\rceil$ for multiplier
$m\in\{0.25,0.5,1,2\}$), fixed ($T{=}12$), streaming ($T{=}n{+}2$).
Horizon curriculum (used for chained $A_5$/$S_5$): the per-position loss
is truncated to the first $h$ positions, $h$ ramping $4\!\to\!n_{\max}$
within each stage. Operator-first $S_5$ curriculum: stages
$2{:}20\text{k},4{:}10\text{k},8{:}10\text{k},16{:}10\text{k},32{:}12\text{k}$
at scale \textsc{m}. Annealing arms warm-start from a converged
checkpoint of the source schedule and re-train under the target schedule
with stages $16{:}6\text{k},32{:}8\text{k}$.

\textbf{Population sizes.} 64 pilot models; 56-run Stage-2 grid
($S_5/S_4/\mathbb{Z}_{60}\times$ scales $\times$ 4 seeds + streaming
arms); expansion to 16 seeds in the three race cells; controls
($A_5$, matched-order; untied; learned-PE; $T$-multiplier;
streaming$\pm$deep supervision; operator-first; annealing); 8
easy-to-hard runs. $\sim$170 trained models in total, on
$8\times$ RTX~3090; individual runs take minutes
(\textsc{s}) to $\sim$1.5h (\textsc{l}, long stages).

\textbf{Easy-to-hard protocol.} Official 32-bit training split
(10{,}000 examples) with random prefix crops $n\in[4,32]$ for the length
curriculum; evaluation on the official 512-bit test split and on
generator-matched intermediate lengths; NoPE, scale \textsc{s}, 12k
steps.

\textbf{Metrics.} Exact match = all supervised positions correct;
per-token accuracy over supervised positions. Generalization outcome in
\S\ref{sec:predict}: mean per-token accuracy at $n\in\{32,64\}$ under the
schedule-matched loop budget, recomputed uniformly from checkpoints.
Statistics: tie-aware Spearman; partial correlations by residual rank
regression on adjacent-loop cosine.

\end{document}